%% file: tact.tex
\documentclass[aos,preprint]{imsart}
\usepackage{amsmath, amssymb, amsthm}
\usepackage{mathrsfs}
\usepackage{fullpage}
\usepackage{graphicx, color,hyperref,epstopdf}
\usepackage[usenames,dvipsnames]{xcolor}
\usepackage{comment}

\newcommand{\boldX}{\mathbf{X}}
\newcommand{\boldY}{\mathbf{Y}}
\newcommand{\clustertree}{\mathcal{C}}
\newcommand{\Ballfamily}{\mathcal{B}}
\newcommand{\Indicator}{\mathbb{I}}

\input{macros}

\newcommand*{\defeq}{\mathrel{\vcenter{\baselineskip0.5ex \lineskiplimit0pt
                     \hbox{\scriptsize.}\hbox{\scriptsize.}}}%
                     =}

\newcommand{\Net}{\calN}

\begin{document}

\title{Cluster Trees on Manifolds}
\runtitle{Cluster Trees on Manifolds}

\begin{aug}
\author{\fnms{Sivaraman} \snm{Balakrishnan}\ead[label=e1]{sbalakri@cs.cmu.edu}},
\author{\fnms{Srivatsan} \snm{Narayanan}\ead[label=e2]{srivatsa@cs.cmu.edu}},
\author{\fnms{Alessandro} \snm{Rinaldo}\ead[label=e4]{arinaldo@cmu.edu}},
\author{\fnms{Aarti} \snm{Singh}\ead[label=e5]{aarti@cs.cmu.edu}}
\and
\author{\fnms{Larry} \snm{Wasserman}\corref{}
\ead[label=e6]{larry@cmu.edu}}

\runauthor{Balakrishnan et al}

\affiliation{
    School of Computer Science and Statistics Department\\
    Carnegie Mellon University
}

\address{
    School of Computer Science\\
    Carnegie Mellon University\\
    Pittsburgh, PA 15213\\
    \printead{e1}\\
    \printead{e2}\\
    \printead{e5}
}

\address{
    Department of Statistics\\
    Carnegie Mellon University\\
    Pittsburgh, PA 15213\\
    \printead{e4}\\
    \printead{e6}
}
\end{aug}
\begin{quote}
In this paper
we investigate the problem of estimating the cluster tree
for a density $f$ supported on or near a smooth $d$-dimensional
manifold $M$ isometrically embedded in $\mathbb{R}^D$.
We analyze a modified version of a $k$-nearest neighbor
based algorithm recently proposed by
Chaudhuri and Dasgupta \cite{chaudhuri10}.
The main results of this paper show that 
under mild assumptions on $f$ and $M$, 
we obtain rates of convergence that depend on $d$ only but not on the ambient dimension $D$.
We also show that similar (albeit non-algorithmic) results can be obtained for kernel density estimators.
We sketch a construction of a sample complexity lower bound instance 
for a natural class of \emph{manifold oblivious} clustering algorithms.
We further briefly consider the \emph{known} manifold case and show that
in this case a spatially adaptive algorithm achieves better rates.
\end{quote}
\input{intro}
\input{assumptions}
\input{main}
\input{main_proof}

\input{lb}
\input{known_manifold}
\input{noise}

\input{kernel}

\input{experiments}

\input{appendix}

\section{Discussion}
In this paper we have shown that simple non-parametric estimators
based on $k$ nearest neighbors and kernel density
estimates are manifold adaptive estimators of the cluster tree.
We have also introduced the problem of cluster tree recovery 
in the presence of noise. Many open questions remain,
particularly regarding the minimax optimal rates of convergence and rates 
of convergence in the tubular noise case which we hope to address in 
future work.

One of the main advantages of the $k$ nearest neighbors based 
estimator is its easy computability.
In the case of \emph{known} manifolds we have shown a 
more general \emph{spatially adaptive}
algorithm achieves better rates and in current work we are trying to
understand the extent to which spatially adaptive estimators can help
when the manifold is unknown.

Finally, simple modifications of 
these simple non-parametric estimators can also be used
as estimators of various geometric properties of
the level sets of the density. We are currently working on 
these extensions.
{
{
\bibliographystyle{unsrt}
\bibliography{biblio}
}
}
\end{document}

%% file: macros.tex
\newtheorem{theorem}{Theorem}
\newtheorem{lemma}[theorem]{Lemma}

\newtheorem{definition}[theorem]{Definition}

\newtheorem{assumption}{Assumption}

\newcommand{\bbE}{\mathbb{E}}

\newcommand{\bbI}{\mathbb{I}}

\newcommand{\bbR}{\mathbb{R}} 
\newcommand{\bbS}{\mathbb{S}}

\newcommand{\calC}{{\cal C}}

\newcommand{\calN}{{\cal N}}

\newcommand{\calX}{{\cal X}}

\let \hat \widehat
\let \tilde \widetilde

\newcommand{\vol}{\operatorname{vol}}
\newcommand{\Reals}{\bbR}
\newcommand{\Sphere}{\bbS}
\newcommand{\dist}{d}
\newcommand{\Ball}{B}

%% file: intro.tex
\section{Introduction}
In this paper, we study the problem of
estimating the cluster tree of a density
when the density is supported on or near a manifold.
Let $\boldX := \{X_1,\ldots, X_n\}$ be a sample drawn 
i.i.d. from a distribution $P$ with density $f$.
The connected components
$\mathbb{C}_f(\lambda)$
of the
upper level set
$\{x: f(x) \geq \lambda\}$ are called
{\em density clusters.}
The collection
${\cal C} =\{\mathbb{C}_f(\lambda): \lambda \geq 0 \}$
of all such clusters is called the
{\em cluster tree} and estimating this cluster tree is referred
to as {\em density clustering}.

The density clustering paradigm is attractive for various reasons. 
One of the main difficulties of clustering is 
that often the true goals of clustering are not clear 
and this makes clusters, and clustering as a task seem poorly defined.
Density clustering however is estimating a well defined population quantity,
making its goal, consistent recovery of the \emph{population} 
density clusters, clear.
 Typically only mild assumptions
 are made on the density $f$ and this allows extremely general shapes and 
 numbers of clusters at each level.
 Finally, the \emph{cluster tree} is an inherently hierarchical object and
 thus density clustering algorithms typically do not require specification of the ``right'' level, rather
 they capture a summary of the density across all levels.

The search for a simple, statistically consistent estimator
of the cluster tree has a long history.
Hartigan \cite{Hartigan81} showed that
the popular single-linkage algorithm is not consistent for
a sample from $\mathbb{R}^D$, with $D > 1$.
Recently,
Chaudhuri and Dasgupta \cite{chaudhuri10}
analyzed an algorithm which is both simple and consistent.
The algorithm finds the connected components of a sequence of
carefully constructed neighborhood graphs.
They showed that, as long as the parameters
of the algorithm are chosen appropriately, 
the resulting collection of connected components
correctly estimates the cluster tree with high probability.

In this paper, we are concerned with the problem
of estimating the cluster tree when the density $f$
is supported on or near a low dimensional manifold.
The motivation for this work stems from the problem of devising and analyzing clustering algorithms with provable performance that can be used in high dimensional applications. When data live in high dimensions, clustering (as well as other statistical tasks) generally become prohibitively difficult due to the curse of dimensionality, which demands a very large sample size. In many high dimensional applications however data is not spread uniformly but
rather concentrates around a low dimensional set. This so-called manifold hypothesis motivates
the study of data generated on or near low dimensional manifolds and the
 study of procedures that can adapt effectively to the intrinsic dimensionality of this data.

\subsection{Contributions}
Here is a brief summary of the main contributions
of this paper:
\begin{enumerate}
\item We show that the simple algorithm studied in
\cite{chaudhuri10} is consistent and has fast rates of
convergence for data on or near a low dimensional
manifold $M$. The algorithm does not
require the user to first estimate $M$
(which is a difficult problem).
In other words, the algorithm adapts to the (unknown) manifold.
\item We show that the sample complexity
for identifying salient clusters is independent of the ambient dimension.
\item We sketch a construction of a sample complexity lower bound instance 
for a natural class of clustering algorithms that we study in this paper.
\item We show that in the \emph{known} manifold case a modified 
\emph{spatially adaptive} algorithm
achieves better rates, similar to the near minimax-optimal rates of \cite{chaudhuri10}.
\item We introduce a framework for studying consistency of clustering when the distribution
is not supported on a manifold but rather, is concentrated near a manifold.
The generative model in this case is that
the data are first sampled from a distribution
on a manifold and then noise is added.
The original data are latent (unobserved).
We show that for certain noise models we can still efficiently recover the cluster tree on the \emph{latent}
samples.
\item We show similar \emph{statistical} results for the level sets of kernel density
estimates for an appropriately chosen bandwidth. \emph{Computing} the level
sets of the kernel density estimate is however a challenging problem that we
do not address in this paper.
\item We present some simulations to confirm our theoretical results.
\end{enumerate}

\subsection{Related Work}
The idea of using probability density functions for clustering dates
back to Wishart \cite{wishart:69}.
\cite{Hartigan81} expanded on
this idea and formalized the notions of high-density clustering, of
the cluster tree and of consistency and fractional consistency of
clustering algorithms.
In particular,
\cite{Hartigan81}
showed that single linkage clustering is consistent
when $D=1$ but is only fractionally consistent when $D>1$.
\cite{SN:10} and \cite{Stuezle03} have also proposed procedures
for recovering the cluster tree. None of these procedures however, come
with the theoretical guarantees given by \cite{chaudhuri10}, which
demonstrated that a generalization of
Wishart's algorithm allows one to estimate parts of the
cluster tree for distributions with full-dimensional support near-optimally
under rather mild assumptions.
This paper forms the starting point for our work and is
reviewed in more detail in the next section.

In the last two decades, much of the research effort involving the use
of nonparametric density estimators for clustering has focused on the
more specialized problems of optimal estimation of the support of the
distribution or of a fixed level set.  However, consistency of
estimators of a fixed level set does not imply cluster tree
consistency, and extending the techniques and analyses mentioned above
to hold simultaneously over a variety of density levels is
non-trivial.  See, e.g.,
\cite{Polonik:95, tsybakov97, Walther:97, Cuevas.Fraiman:97, 
Cuevas2006, RigVer09, MaierHL09, singh09, Rinaldo2010, stability}, 
and references therein.  
%
%
Estimating the cluster tree has more recently been considered in
\cite{kpotufe11} which also gives a simple pruning procedure for
removing spurious clusters.
\cite{Steinwart11,SriperumbudurS12} propose procedures
 for determining recursively the lowest
split in the
cluster tree and give conditions for asymptotic consistency with
minimal assumptions on the density. 

%% file: assumptions.tex
\section{Background and Assumptions}
\label{section::assumptions}
Let $P$ be a distribution supported on
an unknown $d$-dimensional manifold $M$.
We assume that the manifold $M$ is a $d$-dimensional Riemannian
manifold without boundary embedded in a compact set $\mathcal{X}
\subset \Reals^D$ with $d<D$.
We further assume that the volume of the
manifold is bounded from above by a constant, i.e.,
$\vol_d(M)\leq C$.
The main regularity condition
we impose on $M$ is that its condition number be not
too large. The \emph{condition number} of $M$ is $1/\tau$, 
where $\tau$ is the largest number
such that the open normal bundle
about $M$ of radius $r$ is imbedded in $\mathbb{R}^D$
for every $r < \tau$. The condition number controls the curvature
of $M$ and prevents it from being too close to being self-intersecting
(see \cite{Niyogi2006} for a detailed treatment).

The Euclidean norm is denoted by $\| \cdot \|$ and $v_d$ 
denotes the volume of the $d$-dimensional unit ball in $\mathbb{R}^d$.
$B(x,r)$ denotes the full-dimensional ball of radius $r$ centered at
$x$
and $B_M(x,r) \defeq B(x,r) \cap M$. 
For $Z \subset \Reals^d$ and $\sigma > 0$, define $Z_\sigma = Z + B(0,\sigma)$
and $Z_{M,\sigma} = (Z + B(0,\sigma)) \cap M$.
Note that $Z_\sigma$ is full dimensional, while if $Z \subseteq M$ then $Z_{M,\sigma}$ is
$d$-dimensional.

Let $f$ be the density of $P$ with respect to the uniform measure on $M$.
For $\lambda \geq 0$, let $\mathbb{C}_f(\lambda)$ be the collection of connected components of the level set  $\{x \in \calX: f(x) \geq
\lambda \}$ and define the {\it cluster tree} of $f$ to be the hierarchy   $\mathcal{C} = \{ \mathbb{C}_f
(\lambda): \lambda \geq 0 \}$.
For a fixed $\lambda$, any member of
$\mathbb{C}_f(\lambda)$ is a cluster.
For a cluster $C$ its restriction to the
sample $\boldX$ is defined to be $C[\boldX] = C
\cap \boldX$.  The restriction of the cluster tree
$\mathcal{C}$ to $\boldX$ is defined to be
$\mathcal{C}[\boldX] = \{C \cap \boldX: C \in
\clustertree\}$.  Informally, this restriction is a dendrogram-like
hierarchical partition of $\boldX$.

To give finite sample results, following \cite{chaudhuri10}, we define
the notion of salient clusters. Our definitions are slight modifications of
those in \cite{chaudhuri10} to take into account the manifold assumption.

\begin{definition}
Clusters $A$ and $A'$ are
$(\sigma,\epsilon)$ separated if there exists a nonempty $S \subset M$ such that:
\begin{enumerate}
\item Any path along $M$ from $A$ to $A'$ intersects $S$.
\item $\sup_{x \in S_{M,\sigma}} f(x) < (1 - \epsilon) \inf_{x \in A_{M,\sigma} \cup A'_{M,\sigma}} f(x)$.
\end{enumerate}
\end{definition}

Chaudhuri and Dasgupta \cite{chaudhuri10}
analyze a robust single linkage (RSL) algorithm (in Figure \ref{fig::CD}).
An RSL algorithm estimates the connected components
at a level $\lambda$ in two stages. In the first stage, the sample is
\emph{cleaned} by thresholding the $k$-nearest neighbor distance
of the sample points at a radius $r$ and then, in the second stage, the cleaned sample is \emph{connected} at a connection
radius $R$. The connected components of the resulting graph give an estimate of the restriction
$\mathbb{C}_f(\lambda)[\boldX]$. In Section \ref{sec::lb} we prove a sample complexity lower bound
for the \emph{class of RSL algorithms} which we now define.

\begin{definition}
\label{def::rsl}
The \emph{class of RSL algorithms} refers to any algorithm
that is of the form described in the algorithm in Figure \ref{fig::CD} and relying on Euclidean balls,
with any choice of $k$, $r$ and $R$.
\end{definition}
We define two notions of consistency for an estimator $\hat{\mathcal{C}}$ of the
cluster tree:
\begin{definition} [Hartigan consistency]
For any sets $A$, $A' \subset \calX$, let $A_n$ (resp., $A'_n$)
denote the smallest cluster of $\hat{\mathcal{C}}$ containing $A \cap \boldX$
(resp, $A' \cap \boldX$). We say $\clustertree_n$ is consistent
if, whenever $A$ and $A'$ are different connected components
of $\{ x \, : \, f(x) \geq \lambda\}$ (for some $\lambda > 0$),
the probability that $A_n$ is disconnected from $A'_n$ approaches
$1$ as $n \to \infty$.
\end{definition}
\begin{definition} [$(\sigma,\epsilon)$ consistency]
For any sets $A$, $A' \subset \calX$ such that $A$ and $A^\prime$
are $(\sigma,\epsilon)$ separated, let $A_n$ (resp., $A'_n$)
denote the smallest cluster of $\hat{\mathcal{C}}$ containing $A \cap \boldX$
(resp, $A' \cap \boldX$). We say $\clustertree_n$ is consistent
if, whenever $A$ and $A'$ are different connected components
of $\{ x \, : \, f(x) \geq \lambda\}$ (for some $\lambda > 0$),
the probability that $A_n$ is disconnected from $A'_n$ approaches
$1$ as $n \to \infty$.
\end{definition}
The notion of \emph{$(\sigma,\epsilon)$ consistency} is similar that of Hartigan
consistency except restricted to $(\sigma,\epsilon)$ separated clusters $A$ and $A^\prime$,
and typically associated with a finite sample of size $n$.
\cite{chaudhuri10} prove
the following theorem, establishing finite sample bounds for a particular RSL algorithm.
In this theorem there is no manifold and $f$ is a density with respect to the
Lebesgue measure on $\mathbb{R}^D$.
\begin{theorem}
\label{thm::cd}
There is a constant $C$ such that
the following holds. Suppose that we
run the
algorithm in Figure \ref{fig::CD} with
\begin{equation*}
R = \sqrt{2} r~~~\mathrm{and}~~~k = C \left(\frac{D \log n}{\epsilon^2}\right) \log^2 (1/\delta)
\end{equation*}
then with probability
at least $1-\delta$, the algorithm output $\hat{\mathcal{C}}$ is $(\sigma,\epsilon)$ consistent
provided 
$$
\lambda \geq \frac{1}{v_D (\sigma/2)^D} \frac{k}{n} \left( 1 + \frac{\epsilon}{2} \right).
$$
\end{theorem}
The theorem as stated does not explicitly give a sample complexity bound but it is straightforward
to obtain one by plugging in the value for $k$ and solving for $n$ in the inequality that
restricts $\lambda$ to be large enough (as a function of $n$).

In particular, notice that if
$$
n \geq O\left(\frac{D}{\lambda
\epsilon^2 v_D (\sigma/2)^D} \log \frac{D}{\lambda \epsilon^2 v_D
(\sigma/2)^D}\right)
$$
then we can resolve any pair of $(\sigma,\epsilon)$ clusters at level at least $\lambda$.
It is important to note that this theorem does not apply to the setting
when distributions are supported on a lower dimensional set for at least two reasons: 
(1) the density $f$ is singular with respect to the Lebesgue measure on $\mathcal{X}$
and so the cluster tree is trivial, and (2)
the definitions of saliency with respect to $\mathcal{X}$ are typically not satisfied when
$f$ has a lower dimensional support.

\begin{figure}
\fbox{\parbox{6in}{
\begin{enumerate}
\item For each $X_i$, $r_k(X_i) := \inf\{r : B(X_i,r)$ contains $k$ data points$\}$.
\item As $r$ grows from 0 to $\infty$:
\begin{enumerate}
\item Construct a graph $G_{r,R}$ with nodes
$\{X_i: r_k(X_i) \leq r \}$ and edges $(X_i, X_j)$ if $\|X_i - X_j\| \leq R$.
\item Let $\mathbb{C}(r)$ be the connected components of $G_{r,R}$.
\end{enumerate}
\item Denote $\hat{\mathcal{C}} = \{\mathbb{C}(r): r \in [0,\infty)\}$ and return $\hat{\mathcal{C}}$.
\end{enumerate}
\caption{Robust Single Linkage (RSL) Algorithm}
\label{fig::CD}}}
\end{figure}

%% file: main.tex
\section{Clustering on Manifolds}
\label{section::main}
In this section we show that the RSL algorithm can be
adapted to recover the cluster 
tree of a distribution supported on a manifold
of dimension $d < D$ with the rates depending
only on $d$.
In place of the
cluster salience
parameter $\sigma$, our
rates involve a new parameter $\rho$
\begin{equation*}
\rho := \min \left( \frac{3\sigma}{16}, \frac{\epsilon \tau}{72 d}, \frac{\tau}{16} \right).
\end{equation*}
The precise reason for this definition of $\rho$ 
will be clear from the proofs (particularly of Lemma \ref{lemma:ball-volumes}) but for now notice
that in addition to $\sigma$ it is dependent on the condition number $1/\tau$ and deteriorates
as the condition number increases.
Finally, to succinctly present our results we use 
$\mu := \log n + d \log (1/\rho)$.

%

\begin{theorem}
\label{thm:main}
There are universal constants $C_1$ and $C_2$ such that
the following holds.
For any $\delta > 0$, $0 < \epsilon < 1/2$, run the
algorithm in Figure \ref{fig::CD}
on a sample $\boldX$ drawn from $f$,
where the
parameters are set according to the equations
\begin{equation*}
R = 4 \rho~~~\mathrm{and}~~~k = C_1 \log^2 (1/\delta) (\mu/\epsilon^2).
\end{equation*}
Then with probability at
least $1-\delta$, $\hat{\mathcal{C}}$ is $(\sigma,\epsilon)$ consistent.
In particular, the clusters containing
$A[\boldX]$ and $A^{\prime}[\boldX]$, where $A$ and $A^\prime$ are $(\sigma,\epsilon)$ separated, 
are internally connected and mutually disconnected
in $\mathbb{C}(r)$ for $r$ defined by
\begin{equation*}
v_d r^d \lambda = \frac{1}{1 - \epsilon/6} \left( \frac{k}{n} + \frac{C_2 \log (1/\delta)}{n} \sqrt{k\mu}\right)
\end{equation*}
provided
$$\lambda \geq \frac{2}{v_d \rho^d} \frac{k}{n}.$$
\end{theorem}
Before we prove this theorem a few remarks are in order:
\begin{enumerate}
\item To obtain an explicit sample complexity, as in Theorem \ref{thm::cd}, we plug in the
value of $k$ and solve for $n$ from the inequality restricting $\lambda$. The sample complexity
of the RSL algorithm for recovering $(\sigma,\epsilon)$
clusters at level at least $\lambda$
on a manifold $M$ with condition number at most $1/\tau$ is
$$n = O\left(\frac{d}{\lambda \epsilon^2 v_d \rho^d} \log \frac{d}{\lambda \epsilon^2 v_d \rho^d}\right)$$
where $\rho = C \min \left( \sigma, \epsilon \tau/d , \tau \right)$.
Ignoring constants that depend on $d$ the main difference between this result and the result of \cite{chaudhuri10} 
(Theorem \ref{thm::cd}) is that our results only depend on the manifold dimension $d$ and not the ambient dimension
$D$ (typically $D \gg d$). There is also a dependence of our result on $1/(\epsilon\tau)^d$, for $\epsilon \tau \ll \sigma$.
In Section \ref{sec::lb} we sketch the construction of an instance that
suggests that this dependence is not an artifact of our analysis and that
the sample complexity of the class of RSL algorithms is at least $n \geq 1/(\epsilon\tau)^{\Omega(d)}$.
\item Another aspect is that our choice of the connection radius $R$ depends on the (typically) unknown $\rho$, while
for comparison, the connection radius in \cite{chaudhuri10} is chosen to be $\sqrt{2}r$. Under the mild assumption
that $\lambda \leq n^{O(1)}$ (which is satisfied for instance, if the density on $M$ is bounded from above), we show in Section \ref{app::poly}
that an identical theorem holds for $R = 4r$. $k$ is the only real tuning parameter of this algorithm whose choice depends
on $\epsilon$ and an unknown leading constant. 
\item It is easy to see that this theorem also establishes consistency for recovering the entire cluster tree by selecting an
appropriate schedule on $\sigma_n, \epsilon_n$ and $k_n$ that ensures that \emph{all} clusters are distinguished for $n$
large enough (see \cite{chaudhuri10} for a formal proof).
\end{enumerate}

Our proofs structurally mirror those in \cite{chaudhuri10}. We begin with a few
technical results in \ref{sec::tech}. In
Section \ref{sec::sepconnect} we establish $(\sigma,\epsilon)$ consistency by
showing that the clusters are mutually
disjoint and internally
connected. The main technical challenge is that the curvature of the
manifold, modulated by its condition number $1/\tau$, limits our
ability to resolve the density level sets from a finite sample, by
limiting the maximum cleaning and connection radii the algorithm can
use. In what follows, we carefully analyze this effect and show that
somewhat surprisingly, despite this curvature, essentially the same
algorithm is able to adapt to the unknown manifold and produce a
consistent estimate of the entire cluster tree. Similar manifold adaptivity
results have been shown in classification \cite{dasgupta08}
and in non-parametric regression \cite{kpotufe12,bickel}.

%% file: main_proof.tex
\label{section::proofs}

\subsection{Technical results}
\label{sec::tech}
In our proof, we use the uniform convergence
of the empirical mass of Euclidean 
balls to their true mass. In the
full dimensional setting of \cite{chaudhuri10}, this follows from standard 
VC inequalities. To the best of our knowledge however sharp 
(ambient dimension independent) 
inequalities for manifolds are unknown.
We get around this obstacle
by using the insight that, in order to analyze
the RSL algorithms, uniform convergence
for Euclidean balls around the \emph{sample points} and around
a \emph{fixed minimum $s$-net $\Net$ of $M$} (for an
appropriately chosen $s$) suffice to analyze
the RSL algorithm.

Recall, an
$s$-net $\Net \subseteq M$
is such that every point of $M$ is
at a distance at most $s$ from
some point in $\Net$.
Let 
$$\Ballfamily_{n, \Net} := \Big\{ \Ball (z, s) \ :\ z \in \Net \cup \boldX, s \geq 0 \Big\}$$
be the collection of balls whose centers are sample or net points.
We are ready to state our uniform
convergence lemma. The proof is in
Section \ref{apdx::uniform-convergence}. 

\begin{lemma} [Uniform Convergence]
\label{lemma:uniform-convergence}
Assume $k \geq \mu$.
Then there exists a constant $C_0$
such that the following holds. For every $\delta > 0$,
with probability $> 1-\delta$, for all $B\in {\cal B}_{n, \Net}$,
we have:
\begin{eqnarray*}
P(B) \geq \frac{C_\delta \mu}{n} & \implies & P_n(B) > 0 \\ 
P(B) \geq \frac{k}{n} + \frac{C_\delta}{n} \sqrt{k \mu} & \implies & P_n(B) \geq \frac{k}{n} \\
P(B) \leq \frac{k}{n} - \frac{C_\delta}{n} \sqrt{k \mu} & \implies & P_n(B)  < \frac{k}{n},
\end{eqnarray*}
where $C_{\delta} := 2C_0 \log (2/\delta)$, and 
$\mu := 1 + \log n + \log |\Net| = C d + \log n + d \log (1/s)$. Here 
$P_n (B) = |\boldX \cap B| / n$ denotes the empirical probability measure
of $B$, and $C$ is a universal constant.
\end{lemma}

Next we provide a tight estimate of the
volume of a small ball intersected with
$M$. This bounds the distortion of the apparent density
due to the curvature of the manifold and is central to many of 
our arguments. Intuitively,
the claim states that the volume is approximately that of a $d$-dimensional
Euclidean ball, provided that its radius is small
enough compared to $\tau$. The lower bound is
based on Lemma 5.3 of \cite{Niyogi2006} while the upper
bound is based on a modification of the main result of \cite{Chazal2013}. 

\begin{lemma} [Ball volumes]
\label{lemma:ball-volumes} 
Assume $r < \tau/2$. Define $S := \Ball (x, r) \cap M$ for a point $x \in M$. Then
\[
\left( 1 - \frac{r^2}{4\tau^2} \right)^{d/2} v_d r^d \leq \vol_d(S) \leq
v_d \left( \frac{\tau}{\tau - 2 r_1} \right)^d r_1^d ,
\]
where $r_1 = \tau - \tau \sqrt{ 1 - 2r/\tau}$. In particular,
if $r \leq \epsilon \tau/72 d$ for $0 \leq \epsilon < 1$,
then
\[
v_d r^d  (1 - \epsilon/6)  \leq \vol_d(S) \leq v_d r^d (1 + \epsilon/6)
.\]
\end{lemma}

\subsection{Separation and Connectedness}
\label{sec::sepconnect}
\begin{lemma}[Separation] \label{lemma:separation} 
Assume that we pick $k$, $r$ and $R$ to satisfy the conditions:
\begin{eqnarray*}
r  & \leq & \rho \\ 
R & = & 4 \rho \\
v_d r^d (1 - \epsilon/6)  \lambda &\geq& \frac{k}{n} + \frac{C_\delta}{n} \sqrt{k\mu} \\
v_d r^d (1 + \epsilon/6) \lambda (1 - \epsilon) &\leq& \frac{k}{n} - \frac{C_\delta}{n} \sqrt{k\mu}.
\end{eqnarray*}
Then with probability $1-\delta$, we have:
\begin{enumerate}
\item All points in $A_{\sigma-r}$ and $A'_{\sigma-r}$ are kept, and all points
  in $S_{\sigma-r}$ are removed.
\item The two point sets $A \cap \boldX$ and $A' \cap \boldX$ are disconnected in
  $G_{r,R}$.
\end{enumerate}
\end{lemma}

\begin{proof}
The proof is analogous to the separation proof of \cite{chaudhuri10} with several
modifications. Most importantly, we need to ensure that despite the curvature of the manifold
we can still resolve the density well enough to guarantee that we can identify and eliminate points in the
region of separation.

Throughout the proof, we will assume that
the good event in Lemma 
\ref{lemma:uniform-convergence} (uniform convergence for
$\Ballfamily_{n, \Net}$) occurs.
Since $r \leq \epsilon \tau / 72 d$, by Lemma
\ref{lemma:ball-volumes}
$\vol(\Ball_M(x, r))$ is between
$v_d r^d (1 - \epsilon/6)$ and
$v_d r^d (1 + \epsilon/6)$, for any $x \in M$. So if
$X_i \in A \cup A'$,
then $\Ball_M(X_i, r)$ has
mass at least $v_d r^d (1 - \epsilon/6) \cdot \lambda$. Since this is $
\geq \frac{k}{n} + \frac{C_\delta}{n} \sqrt{k\mu}$
by assumption,
this ball contains at least $k$ sample points,
and hence $X_i$ is kept.

On the other hand, if $X_i \in S_{\sigma - r}$,
then the set $\Ball_M(X_i, r)$ contains
mass at most $v_d r^d (1 + \epsilon/6) \cdot \lambda (1 - \epsilon)$. This
is $\leq \frac{k}{n} - \frac{C_\delta}{n} \sqrt{k\mu}$.
Thus by Lemma
\ref{lemma:uniform-convergence} $\Ball_M(X_i, r)$
contains fewer than $k$ sample points, and hence
$X_i$ is removed.

To prove the graph is disconnected, we first need a bound on the
geodesic distance between two points that are at most $R$ apart
in Euclidean distance. Such an estimate follows from Proposition
6.3 in \cite{Niyogi2006} who show that if $\|p-q\| = R \leq \tau/2$, then
the geodesic distance
$$d_M(p,q) \leq \tau - \tau \sqrt{ 1 - \frac{2R}{\tau}}. $$ In particular,
if $R \leq \tau/4$, then $d_M(p,q) < R \left(1 + \frac{4R}{\tau}\right) \leq 2R$. 
Now, notice that if the graph is connected there must
be an edge that connects two points that are at a geodesic
distance of at least $2(\sigma - r)$.
Any path between a point in $A$ and a point in $A^\prime$ along $M$
must pass through $S_{\sigma-r}$ and must have a geodesic length
of at least $2(\sigma-r)$. This is impossible if the
connection radius satisfies
$2R < 2(\sigma - r)$, which follows by the assumptions on
$r$ and $R$.
\end{proof}

All the conditions in Lemma \ref{lemma:separation} can be simultaneously
satisfied by setting $k := 16 C_\delta^2 (\mu/\epsilon^2)$, and
\begin{equation} \label{eqn:set-r}
v_d r^d (1 - \epsilon/6) \cdot \lambda = \frac{k}{n} + \frac{C_\delta}{n} \sqrt{k\mu}.
\end{equation}
The condition on $r$ is satisfied since $$\lambda \geq \frac{2}{v_d \rho^d} \frac{k}{n}$$
and the condition on $R$ is satisfied by its definition.
%

\begin{lemma} [Connectedness]
\label{lemma:connectedness}
Assume that the parameters
$k, r$ and $R$ satisfy the separation
conditions (in Lemma \ref{lemma:separation}).
Then, with probability at least $1 - \delta$,
$A[\boldX]$ is connected in $G_{r,R}$.
\end{lemma}

\begin{proof}
Let us show that any two points in $A \cap \boldX$ are
connected in $G_{r,R}$. Consider $y, y' \in A \cap \boldX$.
Since $A$ is connected, there is a path $P$
between $y, y'$ lying entirely inside $A$, i.e., a
continuous map $P: [0,1] \to A$ such that $P(0) = y$
and $P(1) = y'$. We can find
a sequence of points $y_0, \ldots, y_t \in P$ such that
$y_0 = y$, $y_t = y'$, and the geodesic
distance on $M$ (and hence the Euclidean distance) between
$y_{i-1}$ and $y_i$ is at most $\eta$, for an arbitrarily small
constant $\eta$.

Let $\Net$ be minimal $R/4$-net of $M$. 
There exist
$z_i \in \Net$ such that $\| y_i - z_i \| \leq R/4$.
Since $y_i \in A$, we have $z_i \in A_{M, R/4}$,
and hence the ball $\Ball_M (z_i, R/4)$ lies
completely inside $A_{M, R/2} \subseteq A_{M, \sigma - r}$.
In particular, the density inside the ball
is at least $\lambda$ everywhere, and hence the mass
inside it is at least
$$ v_d (R/4)^d (1 - \epsilon/6) \lambda \geq \frac{C_\delta \mu}{n}.$$
Observe that $R \geq 4r$ and so this condition is satisfied as
a consequence of satisfying Equation \ref{eqn:set-r}.
Thus Lemma \ref{lemma:uniform-convergence} guarantees that
the ball $\Ball_M(z_i, R/4)$ contains at least one sample point,
say $x_i$. (Without loss of generality, we may assume
$x_0 = y$ and $x_t = y'$.) Since the ball lies completely in
$A_{M, \sigma-r}$, the sample point $x_i$ is not removed
in the cleaning step (Lemma \ref{lemma:separation}).

Finally, we bound $d(x_{i-1}, x_i)$ by
considering the sequence of points
$(x_{i-1}, z_{i-1}, y_{i-1}, y_i, z_i, x_i)$.
The pair $(y_{i-1},y_{i})$ are at most $s$
apart and the other successive pairs at most $R/4$
apart, hence
$\dist(x_{i-1}, x_i) \leq 4 (R/4) + \eta = R + \eta$.
The claim follows by letting $\eta \to 0$.
\end{proof}

%% file: lb.tex
\section{A lower bound instance for the class of RSL algorithms}
\label{sec::lb}
Recall that the sample complexity in Theorem \ref{thm:main} scales as
$$n = O\left(\frac{d}{\lambda \epsilon^2 v_d \rho^d} \log \frac{d}{\lambda \epsilon^2 v_d \rho^d}\right)$$
where $\rho = C \min \left( \sigma, \epsilon \tau/d, \tau \right)$.
For full dimensional densities, \cite{chaudhuri10} showed the information
theoretic lower bound
$$n = \Omega \left(\frac{1}{\lambda \epsilon^2 v_D \sigma^D} \log \frac{1}{\lambda \epsilon^2 v_D \sigma^D}\right).$$
Their construction can be straightforwardly modified to a $d$-dimensional instance on a smooth manifold. Ignoring constants
that depend on $d$,
these upper and lower bounds can still differ by a factor of $1/(\epsilon\tau)^d$, for $\epsilon \tau \ll \sigma$. In this
section we provide an informal sketch of a hard instance for the class of RSL algorithms (see Definition \ref{def::rsl}) that suggests 
a sample complexity lower bound of $n \geq 1/(\epsilon\tau)^{\Omega(d)}$.


We first describe our lower bound instance. The manifold $M$
consists of two disjoint components, $C$ and $C'$.
The component $C$ in turn contains three parts,
which we call `top', `middle', and
`bottom' respectively. The middle part, denoted $M_2$,
is the portion of the standard $d$-dimensional unit sphere
$\Sphere^d(0, 1)$ between the planes $x_1 = +\sqrt{1 -4 \tau^2}$ and
$x_1 = - \sqrt{1 - 4 \tau^2}$. The top part, denoted $M_1$, is the upper hemisphere
of radius $2 \tau$ centered at $(+\sqrt{1 - 4\tau^2}, 0, 0, \ldots, 0)$.
The bottom part, denoted $M_3$, is a symmetric hemisphere
centered at $(-\sqrt{1 - 4\tau^2}, 0, 0, \ldots, 0)$.
Thus $C$ is obtained by gluing a portion of the unit 
sphere with two (small) hemispherical caps. $C$ as described 
does not have a condition
number at most $1/\tau$ because of the ``corners'' at the intersection
of $M_2$ and $M_1 \cup M_3$. This can be fixed without
affecting the essence of the construction by
smoothing this intersection by rolling a
ball of radius $\tau$ around it (a similar construction is made rigorous in
Theorem 6 of \cite{Genovese.JMLR}).
Finally, the component $C'$
is a sphere 
far away from $C$ whose function ensure that $f$ integrates to $1$.

Let $P$ be the distribution on $M$ whose density over $C$ is 
$\lambda$ if $|x_1| > 1/2$, and $\lambda(1-\epsilon)$
if $|x_1| \leq 1/2$, where $\lambda$ is chosen small
enough such that $\lambda \vol_d(C) \leq1$. The density
over $C'$ is chosen such that the total mass of the manifold is $1$. Now 
$M_1$ and $M_3$ are $(\sigma, \epsilon)$ separated at level $\lambda$
for $\sigma = \Omega(1)$. The separator set $S$ is the equator 
of $M_2$ in the plane $x_1 = 0$. 


We now provide some intuition for why RSL algorithms will require
$n \geq 1/(\epsilon\tau)^{\Omega(d)}$ to succeed on this instance.
We focus our discussion on RSL algorithms with $k > 2$, i.e. on
algorithms that do in fact use a \emph{cleaning} step, ignoring
the single linkage algorithm which is known to be inconsistent
for full dimensional densities.

Intuitively, because of the curvature of the described instance, 
the mass of
a sufficiently large Euclidean ball in the separator set is \emph{larger}
than the mass of a corresponding ball in the true clusters.
This means that any algorithm that uses large balls cannot reliably
clean the sample and this restricts the size of the balls that can be used.
Now if points in the regions of high density are to survive then
there must be $k$ sample points in the \emph{small} ball around any
point in the true clusters and this gives us a lower bound on the 
necessary sample size.

The RSL algorithms work by counting the
number of sample points inside the balls $B(x, r)$ centered at
the sample points $x$, for some radius 
$r$. In order for the algorithm to reliably 
resolve $(\sigma,\epsilon)$ 
clusters, it should distinguish
 points in the separator set $S \subset M_2$ from those in the
 level $\lambda$ clusters $M_1 \cup M_3$. A necessary 
 condition for this is 
 that the mass of a ball 
$B(x, r)$ for $x \in S_{\sigma - r}$ should be strictly smaller
than the mass inside $B(y, r)$ for $y \in M_1 \cup M_3$. In 
Section \ref{app::capvol}, we show that this   condition restricts the radius $r$ to be at most $O(\tau \sqrt{\epsilon/d})$. 

Now, consider any sample point $x_0$ in $M_1 \cup M_3$ (such an $x$ exists
with high probability). Since $x_0$ should not be
removed during the cleaning step, the ball
$B(x_0, r)$ must contain some other sample point
(indeed, it must contain at least $k-1$ more
sample points). By a union bound, this happens with
probability at most
$$(n-1) v_d r^d \lambda \leq O(d^{-d/2} n \tau^d \epsilon^{d/2} \lambda).$$
If we want the algorithm to succeed with probability at least 1/2 (say) then
$$n \geq \Omega \left(\frac{d^{d/2}}{\tau^d \lambda \epsilon^{d/2}} \right).$$ 


%% file: known_manifold.tex
\section{A modified algorithm for the known manifold case}
In this section we consider the case when the manifold is \emph{known}. In particular, we
assume that we have an oracle that given as input a point $x \in M$ and a number
$V$ returns us a radius $r_x$ such that $\vol_d (\Ball_M(x, r_x)) 
= V$. We call the ball $\Ball (x, r_x)$ the $V$-ball around $x$, and the oracle
a $V$-ball oracle. 
 
Given access to the $V$-ball oracle we show that a modified \emph{spatially adaptive}
RSL algorithm achieves the rate
$$ n \geq O\left( \frac{1}{\lambda v_d \rho^d \epsilon^2} \log \frac{1}{\lambda v_d \rho^d \epsilon^2} \right) $$
where $$\rho \defeq \min \left\{ \frac{\sigma}{10}, \frac{\tau}{16} \right\}$$
In particular, $\rho$ no longer depends on $\epsilon \tau$ and for the case of $\tau$ fixed 
(ignoring constants depending on $d$) the algorithm
achieves the near minimax optimal rates of \cite{chaudhuri10}, in the manifold setting
with $d$ replacing $D$.

 

The modified algorithm is in Figure \ref{fig::CDmod} and it
uses two parameters, $k$ and $V$, to be specified shortly. 

\begin{figure}
\fbox{\parbox{6in}{
\begin{enumerate}
\item For each $X_i$, $r_k(X_i) := \inf\{r : B(X_i,r)$ contains $k$ data points$\}$.
\item As $r$ grows from 0 to $\infty$:
\begin{enumerate}
\item Construct a graph $G_{r,R}$ with nodes
$\{X_i: r_k(X_i) \leq r_{X_i} \}$, where $r_{X_i}$ is the
$V$-ball radius of $X_i$ for $V = v_d r^d$, and edges $(X_i, X_j)$ if $\|X_i - X_j\| \leq R$.
\item Let $\mathbb{C}(r)$ be the connected components of $G_{r,R}$.
\end{enumerate}
\item Denote $\hat{\mathcal{C}} = \{\mathbb{C}(r): r \in [0,\infty)\}$ and return $\hat{\mathcal{C}}$.
\end{enumerate}
\caption{Spatially Adaptive Robust Single Linkage Algorithm}
\label{fig::CDmod}}}
\end{figure}


We begin with a preliminary lemma which is a straightforward consequence 
of Lemma \ref{lemma:ball-volumes}.
\begin{lemma} If $V = v_d r^d$, then $r_l \leq r_x \leq r_u$, where 
\[
r_l := r \left( 1- \frac{6r}{\tau} \right) \text{ and }
r_u := r \left( 1 + \frac{6r}{\tau} \right) 
.\]
\end{lemma}

\begin{theorem}
\label{thm:vball}
There are universal constants $C_1$ and $C_2$ such that
the following holds.
For any $\delta > 0$, $0 < \epsilon < 1/2$, run the
algorithm in Figure \ref{fig::CDmod}
on a sample $\boldX$ drawn from $f$,
where the
parameters are set according to the equations
\begin{equation*}
R = 4 r_u =  r \left( 1 + \frac{6r}{\tau} \right)~~~\mathrm{and}~~~k = C_1 \log^2 (1/\delta) (\mu/\epsilon^2).
\end{equation*}
for $r$ defined by
\begin{equation*}
v_d r^d \lambda = \frac{k}{n} + \frac{C_2 \log (1/\delta)}{n} \sqrt{k\mu}.
\end{equation*}
Then with probability at
least $1-\delta$, $\hat{\mathcal{C}}$ is $(\sigma,\epsilon)$ consistent.
In particular, the clusters containing
$A[\boldX]$ and $A^{\prime}[\boldX]$, where $A$ and $A^\prime$ are $(\sigma,\epsilon)$ separated, 
are internally connected and mutually disconnected
in $\mathbb{C}(r)$ 
provided
$$\lambda \geq \frac{2}{v_d \rho^d} \frac{k}{n}.$$
\end{theorem}
\begin{proof}
The theorem is a straightforward consequence of the following lemma.
\begin{lemma} [Separation and Connectedness] For the 
parameter choices prescribed in the theorem, provided we satisfy the
following
\begin{eqnarray*}
5 r_u \leq  \sigma~~~~&\mathrm{and}&~~~R \leq  \tau/2 \\ 
V \lambda &\geq & \frac{k}{n} + \frac{C_\delta}{n} \sqrt{k \mu} \\
V \lambda (1-\epsilon) &\leq & \frac{k}{n} - \frac{C_\delta}{n} \sqrt{k \mu}.
\end{eqnarray*}
the following properties hold w.p. 
at least $1- \delta$: 
\begin{enumerate}
\item All points in $A_{\sigma-r_u}$ and $A'_{\sigma - r_u}$ are kept, and all 
    points in $S_{\sigma - r_u}$ are removed.
\item The two point sets $A [\boldX]$ and $A' [\boldX]$ are disconnected 
    in the graph $G_{r, R}$.
\item $A[\boldX]$ and $A^\prime[\boldX]$ are internally connected.
\end{enumerate}
\end{lemma}
\begin{proof}
The proof is similar to that of Theorem \ref{thm:main} and we only highlight the differences.
\begin{enumerate} 
\item The $V$-ball around any point $x$ in the manifold has volume {\em 
exactly} $V$ by definition, and hence part (1) is true under the good event 
described in Lemma \ref{lemma:uniform-convergence}.
 In particular notice that
using $V$-balls removes the necessity for estimating the ball volumes.
\item We show part (2) by contradiction. Assume that the graph 
connects a pair 
of points from $A$ and $A'$. Then the connection step guarantees that every 
edge of the path from $A$ to $A'$ is of Euclidean distance $\leq R \leq 
\tau/2$, and hence geodesic distance $\leq 2R$. Therefore, by part 
(1), there must be an edge of (geodesic) length $2(\sigma - r_u)$. This gives us 
a contradiction, provided $2R \leq 2(\sigma - r_u)$. 
\item For part (3) note that $R = 4 r_u \geq 4 r_x$, and hence an 
$R/4$-ball around any net point in $A_{M, R/4}$ contains at least one sample 
point. The rest of the proof is unchanged. 
\end{enumerate}
\end{proof}
As in the proof of Theorem \ref{thm:main}, we 
set the parameters according to $k = C_{\delta}^2 (\mu / \epsilon^2)$, and 
\[
v_d r^d \lambda = \frac{k}{n} + \frac{C_{\delta}}{n} \sqrt{k \mu}.
\]
By our assumption on $\rho$ and $\lambda$, we can see that $r \leq \rho$, and that 
\[
r_u = r \left( 1 + \frac{6r}{\tau} \right) \leq 
\rho \left( 1 + \frac{6\rho}{\tau} \right) \leq 2 \rho. 
\]
Now, setting $R = 4r_u$, we find that the requirements  
$R \leq \tau/2$ and $R + r \leq \sigma$ are automatically satisfied. 
Similarly, 
the final requirement 
\[
v_d r^d \lambda (1-\epsilon) \leq \frac{k}{n} - \frac{C_\delta}{n} \sqrt{k \mu}
\]
is also satisfied because of our choices of $r$ and $k$.
\end{proof}

%% file: noise.tex
\section{Cluster tree recovery in the presence of noise}
\label{section::noise}
So far we have considered the problem of recovering the
cluster tree given samples from a density supported
\emph{on} a lower dimensional manifold. In this section
we extend these results to the more general
situation when we have \emph{noisy} samples concentrated
\emph{near} a lower dimensional manifold. Indeed it can
be argued that the manifold + noise model is a natural and general
model for high-dimensional data.

In the noisy setting, it is clear that we can infer the cluster
tree of the \emph{noisy} density in a straightforward way. 
A stronger requirement would
be consistency with respect to the underlying \emph{latent} 
sample.
Following the literature on manifold estimation
(\cite{Balakrishnan2011,Genovese.JMLR}) we consider
two main
noise models. For both of them, we specify a distribution $Q$
for the noisy sample.

{\bf 1. Clutter Noise:}
We observe data $Y_1,\ldots, Y_n$ from the mixture $$Q := (1-\pi) U + \pi P$$
where $0 < \pi \leq 1$ and $U$ is a uniform distribution on
$\calX$.
 
Denote the samples drawn from $P$ in this mixture 
$$\boldX = \{X_1,\ldots,X_m\}.$$
The points drawn from $U$ are called background clutter.
In this case, we can show:
\begin{theorem}
\label{thm::clutter}
There are universal constants $C_1$ and $C_2$ such that
the following holds.
For any $\delta > 0$, $0 < \epsilon < 1/2$, run the
algorithm in Figure \ref{fig::CD}
on a sample $\{Y_1, \ldots, Y_n\}$,
with
parameters
\begin{equation*}
R := 4 \rho~~~~~k := C_1 \log^2 (1/\delta) (\mu/\epsilon^2).
\end{equation*}
Then with probability at
least $1-\delta$, $\hat{\mathcal{C}}$ is $(\sigma,\epsilon)$ consistent.
In particular, the clusters containing
$A[\boldX]$ and $A^{\prime}[\boldX]$ are internally connected and mutually disconnected
in $\mathbb{C}(r)$ for $r$ defined by
\begin{equation*}
\pi v_d r^d \lambda = \frac{1}{1 - \epsilon/6} \left( \frac{k}{n} + \frac{C_2 \log (1/\delta)}{n} \sqrt{k\mu}\right)
\end{equation*}
provided
$$\lambda \geq \max \left\{ \frac{2}{v_d \rho^d} \frac{k}{n},
\frac{2 v_D^{d/D} (1 - \pi)^{d/D}}{v_d \epsilon^{d/D} \pi}  \left( \frac{k}{n} \right)^{1 - d/D}
 \right\}$$
 where $\rho$ is now slightly modified (in constants), i.e., 
 $\rho := \min \left(\frac{\sigma}{7},  \frac{\epsilon \tau}{72 d}, \frac{\tau}{24} \right)$.
\end{theorem}


{\bf 2. Additive Noise}:
The  data are of the form $Y_i = X_i + \eta_i$
where $X_1, \ldots, X_n \sim P$
,and $\eta_1, \ldots, \eta_n$ are a sample from \emph{any}
bounded noise distribution $\Phi$, with $\eta_i \in B(0,\theta)$.
Note that $Q$ is the convolution of $P$ and $\Phi$, $Q = P \star \Phi$.

\begin{theorem}
\label{thm::additive-noise}
There are universal constants $C_1$ and $C_2$ such that
the following holds.
For any $\delta > 0$, $0 < \epsilon < 1/2$, run the
algorithm in Figure \ref{fig::CD}
on the sample $\{Y_1, \ldots, Y_n\}$ 
with
parameters 
\begin{equation*}
R := 5 \rho~~~~~k := C_1 \log^2 (1/\delta) (\mu/\epsilon^2).
\end{equation*}
Then with probability at
least $1- \delta$, 
$\hat{\mathcal{C}}$ is $(\sigma,\epsilon)$ consistent for 
$\theta \leq \rho \epsilon / 24 d$.
In particular, the clusters containing
$\{Y_i: X_i \in A\}$ and $\{Y_i: X_i \in A^\prime\}$ are internally connected and mutually disconnected
in $\mathbb{C}(r)$ for $r$ defined by
%
$$v_d r^d (1 - \epsilon/12) (1 - \epsilon/6) \lambda =
\frac{k}{n} + \frac{C_\delta}{n} \sqrt{k\mu}$$ if $$\lambda \geq \frac{2}{v_d \rho^d} \frac{k}{n}$$
and
$\theta \leq \rho \epsilon / 24 d$, where 
$$\rho := \min \left(\frac{\sigma}{7}, \frac{\tau}{24}, \frac{\epsilon \tau}{144d}\right).$$ 
\end{theorem}
The proofs for both Theorems \ref{thm::clutter} and \ref{thm::additive-noise} appear
in Section \ref{sec::noisyproofs}.
Notice that in each case we receive samples from a \emph{full}
$D$-dimensional distribution but are still able to achieve
rates independent of $D$ because these distributions are
concentrated around the lower dimensional $M$.
For the clutter noise case we produce a tree that is consistent
for samples drawn from $P$ (which are \emph{exactly} on $M$),
while in the additive noise case we produce a tree on the
observed $Y_i$s which is $(\sigma,\epsilon)$ consistent 
for the \emph{latent} $X_i$s (for $\theta$ small enough).
It is worth noting that in the case of clutter noise we can still
consistently recover the \emph{entire} cluster tree. Intuitively,
this is because the $k$-NN distances for points on $M$ are much
smaller than for clutter points that are far away from $M$. As a
result the clutter noise only affects a vanishingly low level set
of the cluster tree.
In the case of additive
noise with small variance, it
is possible to recover well-separated clusters at
ambient dimension independent rates.
It is also possible to recover the cluster tree in the presence of
general additive noise distributions via deconvolution
\cite{Koltchinskii2000,Balakrishnan2011} but we do not pursue this
approach here.

%% file: kernel.tex
\section{Kernel Density Estimators}
The results of the previous sections have used $k$-nearest neighbors based density estimators.
However, similar (albeit non-algorithmic)
results can be obtained for kernel density estimators.

For the full dimensional cases we consider the usual kernel density estimators
$$
\hat f_h(x) = \frac{1}{nh^D}\sum_{i=1}^n
K\left(\frac{x-X_i}{h} \right).
$$
For the manifold case we consider the following estimator (notice
that unlike the usual kernel density estimate it does not integrate
to 1),
$$
\hat f_h(x) = \frac{1}{nh^d}\sum_{i=1}^n
K\left(\frac{x-X_i}{h} \right).
$$
In each case, $K: \mathbb{R}^D \rightarrow \mathbb{R}$ is a kernel. In each case,
there is an associated population quantity that will be useful. In the full dimensional case
$$ f_h(x) = \frac{1}{h^D} \bbE_{X \sim f} K \left( \frac{x - X}{h} \right)$$
and in the manifold case
$$ f_h(x) = \frac{1}{h^d} \bbE_{X \sim f} K \left( \frac{x - X}{h} \right)$$
As before $\calC(\hat{f}_h)$ denotes the cluster tree
of the kernel density estimate.

\subsection{Assumptions and preliminaries}
%

We will make one of the following assumptions on the kernel:
\begin{assumption}[Bounded support] \hspace*{\fill}
\label{ass::kern}
\begin{itemize}
\item[{\bf [\ref{ass::kern}A]}]  For the case of full-dimensional densities we will
assume the kernel has bounded support and integrates to $1$, i.e.
   $$\{x: K(x) > 0\} \subseteq B(0,1)$$
 and  $$\int_{x \in \bbR^D} K(\|x\|) = 1$$
 Following \cite{Gine2002}, we will further assume that the class of functions
$$\mathcal{F} = \left\{ K \left(\frac{ x - \circ}{h} \right), x \in \mathbb{R}^D, h > 0 \right\}$$
satisfies, for some positive number $A$ and $v$
$$\sup_P \mathcal{N}(\mathcal{F}_h, L_2(P), \epsilon \|F\|_{L_2(P)}) \leq \left( \frac{A}{\epsilon} \right)^v$$
where $\mathcal{N}(T, d, \epsilon)$ denotes the $\epsilon$-covering number of the metric space
$(T, d)$, $F$ is the envelope function of $\mathcal{F}$ and the supremum is taken
over the set of all probability measures on $\bbR^D$. $A$ and $v$ are called the VC characteristics
of the kernel.
\item[{\bf [\ref{ass::kern}B]}] For the case of densities supported on lower-dimensional
manifolds we will assume a particular form for the kernel
    $$K(x) = \frac{\bbI (x \leq 1)}{v_d}$$
    Observe that this kernel also satisfies the VC assumption above.
\end{itemize}
\end{assumption}
The first assumption is quite mild and can be further relaxed to include kernels with an
appropriate tail decay, albeit at the cost of more complicated proofs. The second assumption
allows us to avoid dealing with integrals over the manifold but can also be similarly relaxed.

\begin{assumption}[Bandwidth regularity: BR(m)]
\label{ass::bwreg}
For some $c > 0$,
\begin{eqnarray*}
h_n \searrow 0, ~~~\frac{nh_n^m}{|\log h_n|} \rightarrow \infty ~~~ \frac{|\log h_n|}{\log \log n} \rightarrow \infty ~~~ \mathrm{and}~~~ h_n^m \leq c h_{2n}^m
\end{eqnarray*}
\end{assumption}

We will first state two preliminary results showing the uniform
consistency of the kernel density estimate.

The first Lemma appears in a similar form in \cite{Rinaldo2010} (Proposition 9) and
is a modification of a result of \cite{Gine2002} (Corollary 2.2). The proof is omitted.
\begin{lemma} [Full dimensional density]
\label{lem::ggfull}
Given $n$ samples from a distribution which has a bounded density $f$ with respect to the
Lebesgue measure on $\bbR^D$
\begin{enumerate}
\item For $n \geq n_0$, where $n_0$ is a constant
depending only on the VC characteristics of $K, \|K\|_\infty, \|K\|_2$ and $f_{\max}$,
and fixed $h \leq h_0$ depending only on $\|K\|_\infty$ and $f_{\max}$
there is a constant $C$
depending on $K$ such that
$$P \left(\|\hat{f}_h - f_h\|_\infty \geq C^\prime \cdot C \sqrt{\frac{f_{\max} \log(1/h)}{nh^D}} \right) \leq \left( \frac{1}{h} \right)^{C^\prime}$$
for any large enough constant $C^\prime$ depending on $K$ and $f_{\max}$ of our choice.
\item For any sequence $h_n \leq h_0$ as before,
satisfying Assumption \ref{ass::bwreg}, BR($D$), for all $n \geq n_0$
as before
$$P  \left(\|\hat{f}_{h_n} - f_{h_n}\|_\infty\geq C^\prime \cdot C \sqrt{\frac{f_{\max} \log(1/h_n)}{nh_n^D}}
\right) \leq \left( \frac{1}{h} \right)^{C^\prime}$$
\end{enumerate}
\end{lemma}

For the ball kernel of Assumption \ref{ass::kern} a similar result holds for densities supported
on a lower dimensional manifold.
\begin{lemma} [Manifold case]
\label{lem::ggmanifold} 
Given $n$ samples from a distribution supported on a smooth Riemannian manifold $M$ with
condition number at most $1/\tau$ with bounded density $f$ with respect to the
uniform measure on $M$
\begin{enumerate}
\item For $n \geq n_0$, where $n_0$ is a constant
depending only on the VC characteristics of $K, \|K\|_\infty, \|K\|_2$ and $\|f\|_\infty$,
and fixed $h \leq \min(\frac{\tau}{8}, h_0)$ where $h_0$
depends only on $\|K\|_\infty$ and $\|f\|_\infty$
there is a constant $C_\delta$ depending on $\delta$ and $n_0$ such that
$$P \left(\|\hat{f}_h - f_h\|_\infty \geq 
C^\prime \cdot C \sqrt{\frac{f_{\max} \log(1/h)}{nh^d}} \right)
\leq \left( \frac{1}{h} \right)^{C^\prime}$$
\item For any sequence $h_n \leq \min(\frac{\tau}{8},h_0)$ as before,
satisfying Assumption \ref{ass::bwreg}, BR($d$), for all $n \geq n_0$
as before
$$P  \left(\|\hat{f}_{h_n} - f_{h_n}\|_\infty \geq
C^\prime \cdot C \sqrt{\frac{f_{\max} \log(1/h_n)}{nh_n^d}} \right)
\leq \left( \frac{1}{h} \right)^{C^\prime}$$
\end{enumerate}
\end{lemma}
\begin{proof}
The proof follows along the lines of those in \cite{Rinaldo2010,Gine2002}. The main modification
to achieve $d$ rates involves a more careful calculation of the variance.

To apply Talagrand's inequality in the proof of \cite{Gine2002}
we need to bound $$\sup_{g \in \mathcal{F}} \mathrm{Var}_f g$$
$\mathcal{F}$ is the set of kernel functions with various bandwidths, and centers
anywhere on $M$.

Let us show how to bound $\sup_{g \in \mathcal{F}_h} \mathrm{Var}_f g$ for a single
bandwidth $h$.
%
\begin{eqnarray*}
 \mathrm{Var}_{X \sim p} \left( K\left( \frac{x - X}{h} \right) \right) & = & \mathbb{E}_{X} \left[ K \left( \frac{x - X}{h} \right) - \mathbb{E}_{X} K \left( \frac{x - X}{h} \right) \right]^2 \\
 & \leq & \left[ \mathbb{E}_{X} K^2 \left( \frac{x - X}{h} \right) \right] \\
 & = & \int_{X \in M} K^2 \left( \frac{x - X}{h} \right) f(X) dX \\
 & \leq & \|K\|^2_\infty \int I(X \in B(x,h)) f(X) dX \\
 & \leq & h^d C_d \|K\|^2_\infty \|f\|_\infty
\end{eqnarray*}
The last step follows if $h \leq \frac{\tau}{8}$ by the ball volume Lemma \ref{lemma:ball-volumes}.
Notice that the variance does not depend on $x$ and so the bound holds uniformly over all $x$ on $M$.

Replacing this bound on the variance in the proof of \cite{Gine2002} we obtain the desired result.
\end{proof}


\subsection{Rates of convergence for the cluster tree}
Our first result mirrors the main result of \cite{chaudhuri10}.
\begin{theorem} [Full dimensional cluster tree]
\label{thm::kerntreefull}
There is a constant $C_\delta$ depending on the VC characteristics of the kernel,
$\|K\|_\infty, \|K\|_2, \|f\|_\infty$ and $\delta$
such that the following holds with probability
at least $1 - \delta$,
$\calC(\hat{p}_\sigma)$
is $(\sigma,\epsilon)$ consistent for any pair of clusters $A$, $A^\prime$
at level at least $\lambda$ for
$$n \geq \frac{C_\delta}{ \sigma^D \lambda^2 \epsilon^2} \log \left( \frac{1}{\sigma} \right) $$
\end{theorem}
Notice, in particular that while for the $k$-nearest neighbor based algorithm the choice
of $k$ depends on $\epsilon$ for the kernel density estimate the optimal
choice of bandwidth depends on $\sigma$. Also notice unlike the result of \cite{chaudhuri10}
this result requires the density to be uniformly upper bounded.
\begin{proof}
To prove this theorem it suffices to show that the regions $A$ and $A^\prime$ are internally
connected and mutually separated.

Let us first show that $\sigma$-clusters $A$ and $A^\prime$ (for any $\lambda, \epsilon > 0$) are
connected and separated in $\mathcal{C}( f_\sigma)$.
Consider any point $x \in A \cup A^\prime$,
$$f_\sigma(x) = \int_{y \in B(x,\sigma)} K \left(\frac{y-x}{h} \right) f(y) dy \geq \lambda  \int_{y \in B(x,\sigma)} K \left(\frac{y-x}{h} \right) dy \geq \lambda$$
Similarly, we can see that for any point in the separator $S$, $f_\sigma(x) < \lambda(1-\epsilon)$.
In particular, $\sigma$-clusters $A$ and $A^\prime$ are distinguished in $\mathcal{C}( f_\sigma)$ at level
$\lambda$ as desired.

Now, we use Lemma \ref{lem::ggfull}. Notice for a constant $C_\delta$
$$n \geq \frac{C_\delta}{ \sigma^D \lambda^2 \epsilon^2} \log \left( \frac{1}{\sigma} \right)$$
we have $$\|\hat{f}_\sigma - f_\sigma\|_\infty \leq \frac{\lambda \epsilon}{2}$$ with probability
$1 - \delta$. Let $\mathcal{E}_1$ denote the event $\{\|\hat{f}_\sigma - f_\sigma\|_\infty \leq \frac{\lambda \epsilon}{2}\}$.

Now, let us consider the cluster tree of $\hat{f}_\sigma$ at level $\lambda - \frac{\lambda \epsilon}{2}$.
On $\mathcal{E}_1$, for any point $x \in A \cup A^\prime$ we know $f_\sigma \geq \lambda$ and thus
$\hat{f}_\sigma \geq \lambda - \frac{\lambda\epsilon}{2}$. Similarly for $x \in S$ we have
$\hat{f}_\sigma < \lambda - \frac{\lambda\epsilon}{2}$. These together show that
on $\mathcal{E}_1$ $A$ and $A^\prime$ are distinguished in $\mathcal{C}(\hat{f}_\sigma)$
at level $\lambda - \frac{\lambda \epsilon}{2}$.
This establishes the theorem.
\end{proof}
To establish Hartigan consistency we select a schedule $h_n$ satisfying Assumption \ref{ass::bwreg}.
Under mild conditions connected components of any level set at $\lambda$, are $(\sigma,\epsilon)$ separated
for some $\sigma, \epsilon > 0$ and are distinguished for $n$ large enough.

We can similarly give a manifold version of this result. Define $$\rho = \min \left(\sigma, \frac{\tau}{8}, \frac{\epsilon \tau}{72d} \right)$$
\begin{theorem} [Cluster tree on manifolds]
\label{thm::kerntreemanifold}
There is a constant $C_\delta$ depending on the VC characteristics of the kernel,
$\|K\|_\infty, \|K\|_2, \|f\|_\infty$ and $\delta$
such that the following holds with probability
at least $1 - \delta$, for all $\epsilon \leq 1/2$
$\calC(\hat{p}_\rho)$
is $(\sigma,\epsilon)$ consistent for any pair of clusters $A$, $A^\prime$
at level at least $\lambda$ for
$$n \geq \frac{C_\delta}{ \rho^D \lambda^2 \epsilon^2} \log \left( \frac{1}{\rho} \right) $$
\end{theorem}
\begin{proof}
Let us again consider $f_\rho$. For any point $x \in A \cup A^\prime$,
$$f_\rho(x) = \frac{1}{h^d} \bbE_{X \sim f} K \left( \frac{x - X}{h} \right) = \frac{1}{v_d \rho^d} \int_{X \in B_M(x,h)} dX
\geq \lambda \left(1 - \frac{\epsilon}{6} \right) $$
where the second equality follows from the assumed form of the kernel, and the inequality follows from Lemma \ref{lemma:ball-volumes} under the assumption on $\rho$.
Similarly, for any point in $S$ we have
$$f_\rho(x) < \lambda \left(1 - \epsilon\right)\left(1 + \frac{\epsilon}{6} \right)$$
The gap between these is at least $\lambda \epsilon/2$, and hence $A$ and $^\prime$ are distinguished
in $f_\rho$ at level $\lambda (1 - \epsilon/6)$.

The proof that these clusters are distinguished in $\hat{f}_\rho$ follows from an identical argument to the one in the proof of Theorem \ref{thm::kerntreefull}, replacing the use of Lemma \ref{lem::ggfull} with Lemma \ref{lem::ggmanifold}.
\end{proof}


%
%
%
%

%% file: experiments.tex
\section{Simulations}
\label{sec::tuning}
\begin{figure}
  \centering
  \includegraphics[width=5.6in, height=2.4in]{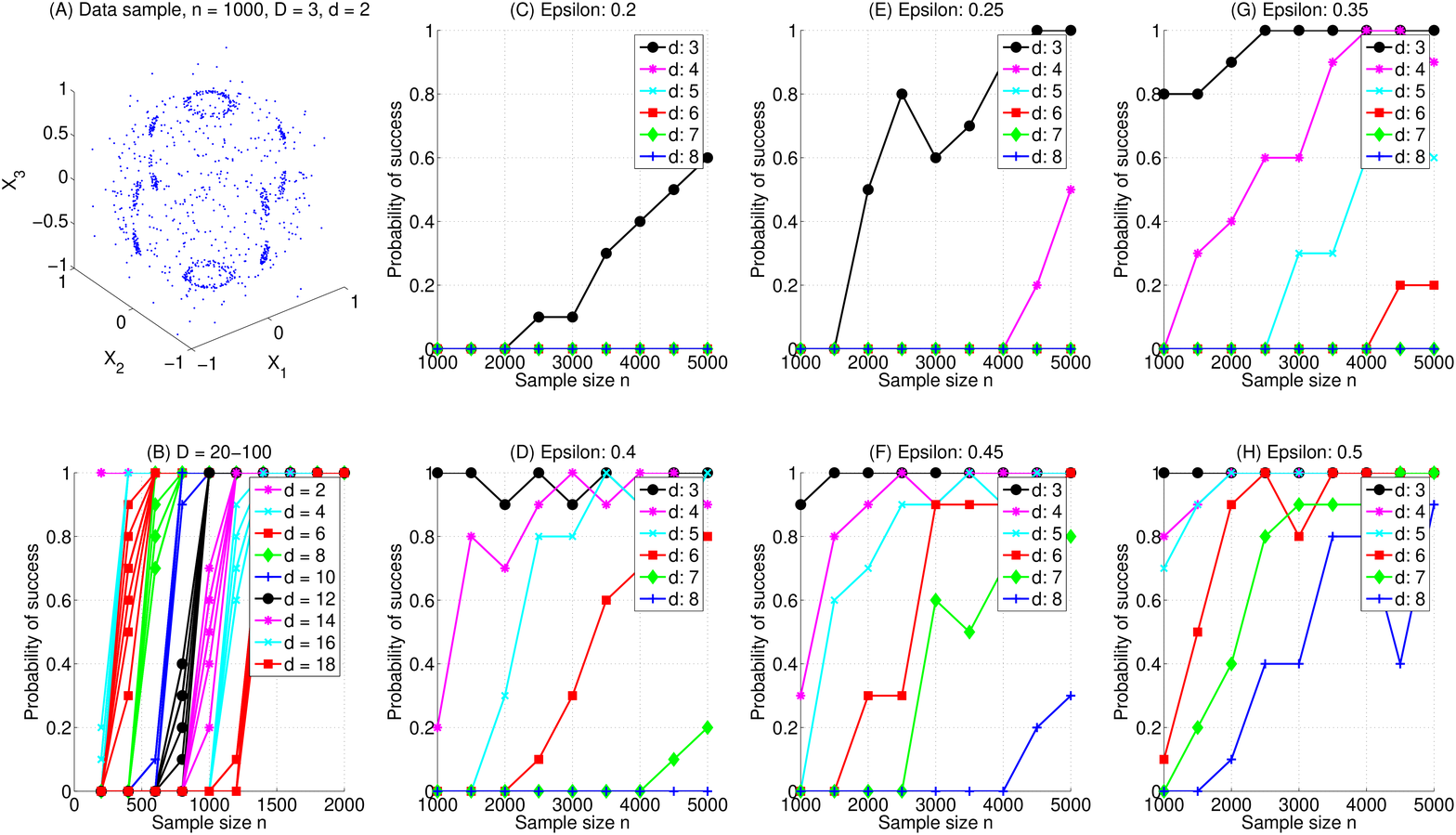}
  \caption{Figures show the average probability of success across 10 trials for different $(n, d, D, \epsilon)$.}
  \label{fig::main}
\end{figure}
Figure \ref{fig::main} depicts the results of simulations we performed to test our 
main theoretical predictions. For Figure \ref{fig::main}(B) we sample data from a mixture distribution
on a unit $d$-sphere. The mixture has 10 salient clusters (with a total mixture weight of $0.7$) 
mixed with uniform samples on the sphere with mixture weight $0.3$. Finally, we mix samples
from this density with $D$-dimensional clutter noise with $\pi = 0.8$. 
A sample is shown in Figure \ref{fig::main}(A) for $d=2, D = 3$ and $n = 1000$.
For Figures \ref{fig::main}(C)-(H) we simulate data from the lower bound instance described in Section \ref{sec::lb}.

In Figure \ref{fig::main}(B), we plot the probability of successfully
recovering the 10 clusters in the cluster tree as a function of sample size. The
figure confirms that the sample size is independent of the
ambient dimension $D$ but (typically) gets worse with the manifold dimension $d$. In particular,
the figure shows that for $D=\{20,40,60,80,100\}$ (in the same color) sample complexities are nearly unchanged.
Figures \ref{fig::main}(C)-(H), shows the effect on sample size of $(\epsilon,d)$ for the lower
bound instance. Notice, that for a fixed $\epsilon$ and $n$ the probability of success decays rapidly with
increasing $d$ and that for a fixed $d$ and $n$ the probability of success grows with $\epsilon$, in agreement
with our $1/\epsilon^{\Omega(d)}$ prediction and in contrast to the $1/\epsilon^2$ scaling predicted 
by \cite{chaudhuri10} for recovering a full-dimensional cluster tree.

%% file: appendix.tex
\section{Additional proofs}
In this section we first prove some technical lemmas before giving full proofs
of various claims made in the paper.
\subsection{Volume estimates for small balls on manifolds}
\begin{theorem}
 \label{lem:ballvol}
If $$r \leq \frac{\epsilon \tau}{12 d}$$
for $0 \leq \epsilon < 1$ then
$$ v_dr^d  \left( 1 - \epsilon \right)  \leq \mathrm{vol}(S) \leq v_d r^d \left( 1 + \epsilon \right) $$
\end{theorem}
\begin{proof}
The lower bound follows from \cite{Niyogi2006} (Lemma 5.3) who show that
\begin{lemma}
For $r < \frac{\tau}{2}$
$$\mathrm{vol}(S) \geq \left( 1 - \frac{r^2}{4\tau^2} \right)^{d/2} v_d r^d$$
\end{lemma}
The upper bound follows from \cite{Chazal2013} who shows that
\begin{lemma}
For $r < \frac{\tau}{2}$
$$\mathrm{vol}(S) \leq v_d \left( \frac{\tau}{\tau - 2\alpha} \right)^d \alpha^d $$
where $$\alpha = \tau - \tau \sqrt{ 1 - \frac{2r}{\tau}}$$
\end{lemma}
To produce the result of the theorem we will need some careful manipulation of these two lemmas.
In particular, we need the following estimates
\begin{lemma}
\label{lem:taylor}
$$f(x) = (1 - x)^{1/2} \geq 1 - \frac{x}{2} - x^2$$
if $0 \leq x \leq \frac{1}{2}$.
$$f(x) = (1 + x)^{n} \leq 1 + 2nx$$
if $0 \leq x \leq \frac{1}{2n}$.
$$f(x) = (1-x)^{-1} \leq 1 + 2x$$
if $0 \leq x \leq 1/2$
$$f(x) = (1-x)^{n} \geq 1 - 2nx$$
if $0 \leq x \leq \frac{1}{2n}$
\end{lemma}
The proof of this lemma is straightforward based on approximations via Taylor's series and we omit them.

Using Lemma \ref{lem:taylor} we have $$\alpha \leq r \left( 1 + \frac{4 r}{\tau} \right) $$
if $r \leq \frac{\tau}{4}$. Now, using this also notice that

\begin{eqnarray*}
\frac{\tau}{\tau - 2\alpha} \leq \frac{1}{1 - \frac{2r}{\tau} \left( 1 + \frac{4r}{\tau} \right)} \leq 1 + \frac{4r}{\tau} \left( 1 + \frac{4r}{\tau} \right) \\
\end{eqnarray*}
where the second inequality follows from Lemma \ref{lem:taylor} if $r \leq \tau/8$.

Combining these we have the following:

For all $r \leq \frac{\tau}{8}$
$$v_d r^d  \left( 1 - \frac{r^2}{4\tau^2} \right)^{d/2} \leq \mathrm{vol}(S) \leq v_d r^d \left( 1 + \frac{6r}{\tau} \right)^{d} $$

The final result now follows another application of Lemma \ref{lem:taylor} on either side of this inequality.
\end{proof}

\subsection{Bound on covering number}
We need the following bound on the covering number of a manifold. See 
\cite{Niyogi2006} (p. 16) for a proof. 
\begin{lemma}
\label{lemma:coversize}
For $s \leq 2\tau$, the $s$-covering number of $M$ is at most
\[ 
\frac{\vol_d(M)}{\cos^d (\arcsin(s/4\tau)) v_d (s/2)^d} \leq O\left( \frac{\vol_d(M) c^d }{ v_ds^d } \right)
\]
for an absolute constant $c$. In particular, if $\vol_d(M)$ is bounded above by a constant, the $s$-covering 
number of $M$ is at most $O(c^d/ (v_ds^d))$.
\end{lemma}
\begin{proof}
We prove only the second claim. For $s \leq 2 \tau$, we have $\arcsin(s/4 \tau) \leq \pi/6$, and 
hence $\cos (\arcsin(s/4 \tau) ) \geq \sqrt{3}/2$. Plugging this in the bound, we get
\[ 
|\Net| \leq \frac{\vol_d(M) (2 / \sqrt{3})^d }{v_d (s/2)^d},
\]
which gives the claim with $c = 4 / \sqrt{3}$.
\end{proof}

\subsection{Uniform convergence} \label{apdx::uniform-convergence}
In this subsection, we prove uniform convergence for balls 
centered on sample and net points (Lemma \ref{lemma:uniform-convergence}). Consider the
family of balls centered at a fixed point $z$,
$\Ballfamily_z := \Big\{ \Ball (z, s) \ : \ s \geq 0  \Big\}$.
This collection has VC dimension $1$. Thus with probability
$1 - \delta'$, it holds that for every $B \in \Ballfamily_z$, we have
\[
\max \Big\{ \frac{P(B) - P_n(B)}{\sqrt{P(B)}} , \frac{P(B) - P_n(B)}{\sqrt{P_n(B)}} \Big\}
 \leq 2 \sqrt{\frac{\log(2n) + \log (4 / \delta')}{n}} , \\
\]
where $P(B)$ is the true mass of $B$, and $P_n(B) = |\boldX \cap B| / n$ is its
empirical measure. By a union bound over all $z \in \Net$, setting
$\delta' := \delta / (2 |\Net|)$, the
following holds uniformly for every
$z \in \Net$ and every $B \in \Ballfamily_z$ with probability
$1-\delta/2$:
\[
\max \Big\{ \frac{P(B) - P_n(B)}{\sqrt{P(B)}} , \frac{P(B) - P_n(B)}{\sqrt{P_n(B)}} \Big\}
 \leq 2 \sqrt{\frac{\log(2n) + \log (8 |\Net| / \delta)}{n}} .
\]

To provide a similar uniform convergence result for balls
centered at a sample point $X_i$, we consider the
$(n-1)$-subsample $X^{n-1}_{i}$ of $\boldX$ obtained
by deleting $X_i$ from the sample. Let $P^{n-1}_{i}$ be
the empirical probability measure of this subsample:
\[
P_{n-1} (B) := \frac{1}{n-1} \sum_{j \neq i} \Indicator [X_i \in B].
\]
It is easy to check that $P_{n-1}$ is uniformly close
to $P_n$. In particular, for every set $B$ containing
$X_i$, we have
\begin{equation} \label{eqn:Pn-1-and-Pn}
P_{n-1} (B) \leq P_n(B) \leq P_{n-1} (B) + \frac{1}{n}.
\end{equation}
Now, with probability at least $1 - \delta/(2n)$, for any ball $B$ centered
at $X_i$,
\begin{align*}
P(B) - P_{n-1}(B) &\leq 2 \sqrt{\frac{\log(2n-2) + \log  8n/\delta}{n-1}} \cdot \sqrt{P(B)}, \\
P_{n-1}(B) - P(B) &\leq 2 \sqrt{\frac{\log(2n-2) + \log 8n/\delta}{n-1}} \cdot \sqrt{P_{n-1}(B)}.
\end{align*}
Using (\ref{eqn:Pn-1-and-Pn}), we get
\begin{align*}
P(B) - P_n(B) &\leq 2 \sqrt{\frac{\log(2n-2) + \log 8n/\delta}{n-1}} \cdot \sqrt{P(B)}, \\
P_n(B) - P(B) &\leq 2 \sqrt{\frac{\log(2n-2) + \log 8n/\delta}{n-1}} \cdot \sqrt{P_n(B)} + \frac{1}{n}.
\end{align*}
By a union bound over all $X_i \in \boldX$, we get
the claimed inequalities for all sample points
with probability $1-\delta/2$.

Putting together our bounds for balls around sample and
net points, with probability at least $1-\delta$, it holds
that for all $B \in \Ballfamily_{n, \Net}$, we have
\begin{align*}
P(B) - P_n(B) &\leq O \Big( \sqrt{\frac{\mu + \log (1/\delta)}{n}} \Big) \cdot \sqrt{P(B)}, \\
P_n(B) - P(B) &\leq O \Big( \sqrt{\frac{\mu + \log (1/\delta)}{n}} \Big) \cdot \sqrt{P_n(B)} + \frac{1}{n}.
\end{align*}
for $\mu = 1 + \log n + \log |\Net| = O(d) + \log n + d \log (1/s)$ 
(using Lemma \ref{lemma:coversize}). The lemma now follows using simple 
manipulations of these inequalities (see \cite{chaudhuri10} for details).

\subsection{Sketch of the lower bound instance}
\label{app::capvol}

The following lemma gives an estimate of the volume of the
intersection of a small ball with a sphere.

\begin{lemma}[Volume of a spherical cap]
\label{lemma:spherecap-volume}
Suppose $\Sphere^d$ is a $d$-dimensional sphere of radius $\tau$ (embedded in 
$\Reals^{d+1}$), and let $x \in \Sphere^d$. Then, for small enough $r$, 
it holds that
\[
\vol_d(\Ball (x, r) \cap \Sphere^d) = v_d r^d \left( 1 - c_d \frac{r^2}{\tau^2} + O_d\Big( \frac{r^4}{\tau^4} \Big) \right)
\]
where $c_d := \frac{d(d-2)}{8(d+2)}$. Note that $c_1 < 0$, $c_2 = 0$, and $c_d > 0$ for all
$d \geq 3$.
\end{lemma}

In this section, we prove Lemma \ref{lemma:spherecap-volume}. 
The height $h$ of the cap can be easily
checked to be equal to $h = r^2 / 2 \tau$. Now, the volume of the cap is given by
the formula
\[
v_{cap} = \frac{\pi^{(d+1)/2} \tau^d}{\Gamma((d+1)/2)} I_{\alpha} (d/2, 1/2)
\]
where the parameter $\alpha$ is defined by
\[
\alpha := \frac{2\tau h - h^2}{\tau} = \frac{r^2}{\tau^2} ( 1 - \frac{r^2}{4 \tau^2})
.\]
Further $I_\alpha(\cdot, \cdot)$ represents the incomplete beta function:
\begin{align*}
I_{\alpha}(z,w) &= \frac{B(\alpha; z, w)}{B(z,w)}
\\
&= \frac{\int_0^\alpha u^{z-1} (1-u)^{w-1} du}{B(z,w)}
\\
&= \frac{\Gamma(z+w)}{\Gamma(z) \Gamma(w)} \int_0^\alpha u^{z-1} (1-u)^{w-1} du
\end{align*}
Thus,
\begin{align*}
v_{cap}
&= \frac{\pi^{(d+1)/2} \tau^d}{\Gamma((d+1)/2)} \cdot
   \frac{\Gamma((d+1)/2))}{\Gamma(d/2) \Gamma(1/2)} \cdot \int_0^\alpha u^{d/2-1} (1-u)^{-1/2} du
\\ &= \frac{\pi^{d/2} \tau^d}{\Gamma(d/2)} \int_0^\alpha u^{d/2-1} (1-u)^{-1/2} du
\\ &= \frac{d v_d \tau^d }{2} \int_0^{\alpha} u^{d/2-1} (1-u)^{-1/2} du .
\end{align*}
Since $\alpha \to 0$ as $r \to 0$, we can approximate the integral by expanding the
integrand as a Taylor series around $0$:
\begin{align*}
v_{cap}
&= \frac{d v_d \tau^d }{2} \int_0^{\alpha} u^{d/2-1} \Big(1+u/2 + O(u^2) \Big) du
\\ &= \frac{d v_d \tau^d }{2} \left( \frac{\alpha^{d/2}}{d/2} + \frac{1}{2} \frac{\alpha^{d/2+1}}{d/2+1} + O(\alpha^{d/2+2}) \right)
\\ &= v_d \tau^d \alpha^{d/2} \left( 1 + \frac{d}{2(d+2)} \alpha + O(\alpha^2)) \right)
\end{align*}

Finally, using $\alpha := \frac{r^2}{\tau^2} (1 - \frac{r^2}{\tau^2})$, we get
\begin{align*}
v_{cap}
&= v_d r^d \left( 1 - \frac{r^2}{4\tau^2}\right)^{d/2} \left( 1 +  \frac{dr^2}{2(d+2) \tau^2} + O \Bigl( \frac{r^4}{\tau^4} \Bigr) \right)
\\
&= v_d r^d \cdot \left(1 - \frac{dr^2}{8\tau^2} + \frac{d r^2}{2(d+2) \tau^2}  + O_d \Bigl( \frac{r^4}{\tau^4} \Bigr)\right) ,
\end{align*}
which simplifies to the claimed estimate.

We now show that it must be the case that $r \leq O(\tau \sqrt{\epsilon/d})$. 
We argued that for the algorithm to reliably resolve the $(\sigma, \epsilon)$ separated
clusters $M_1$ and $M_3$, an $r$-ball around a sample point in $S_{\sigma - r}$
must have mass appreciably smaller than those around points in $M_1$.  By the 
previous lemma, the two kinds of balls have volumes  
\[ v_d r^d \left( 1 - c_d \frac{r^2}{1^2} + O_d \Big( \frac{r^4}{1^4} \Big) \right) = 
v_d r^d \left( 1 - c_d r^2 + O_d (r^4) \right) \]
and 
\[ 
v_d r^d \left( 1 - c_d \frac{r^2}{4\tau^2} + O_d \Big( \frac{r^4}{16 \tau^4} \Big) \right) 
=
v_d r^d \left( 1 - c_d \frac{r^2}{4\tau^2} + O_d \Big( \frac{r^4}{\tau^4} \Big) \right)  .
\]
Thus we must have 
\[
v_d r^d v_d r^d \Big( 1 - c_d r^2 + O_d (r^4)\Big) 
\cdot  \lambda (1-\epsilon)
\leq v_d r^d \left( 1 - c_d \frac{r^2}{4\tau^2} + O_d \Big( \frac{r^4}{\tau^4} \Big) \right) \cdot \lambda
.\]
This implies that $r^2 \leq O\left( \frac{4\tau^2 \epsilon}{(1 - 4 \tau^2) c_d} \right)$. Hence if $\tau \leq 1/4$, we have
$r \leq \tau \sqrt{\epsilon / c_d}$. Plugging in $c_d = \Omega(d)$ gives us the claim. 


\subsection{Clustering with noisy samples}
\label{sec::noisyproofs}
\subsection{Proof of Theorem \ref{thm::clutter}}
As before we begin by showing separation followed by a proof of connectivity. 
Recall that  
$\rho := \min \left(\frac{\sigma}{7},  \frac{\epsilon \tau}{72 d}, \frac{\tau}{24} \right)$. 

\begin{lemma} [Separation] \label{lemma:noisy-separation}
Assume that we pick $k$, $r$ and $R$ to satisfy the conditions:
\begin{align*}
r \leq \rho,& ~~~~ R = 4\rho  \\
\pi \cdot v_d r^d (1 - \epsilon/6) \cdot \lambda &\geq \frac{k}{n} + \frac{C_\delta}{n} \sqrt{k\mu}, \\
\pi \cdot v_d r^d (1 + \epsilon/6) \cdot \lambda (1 - \epsilon) + (1 - \pi) \cdot v_D r^D &\leq \frac{k}{n} - \frac{C_\delta}{n} \sqrt{k\mu}.
\end{align*}
Then with probability $1-\delta$, it holds that:
\begin{enumerate}
\item All points in $A_{M, \sigma-r}$ and $A'_{M, \sigma-r}$ are kept, and all
points in $\calX \setminus M_r$ and $S_{\sigma-r}$ are removed. Here, $M_r$ is
  the tubular region around $M$ of width $r$.

\item The two point sets $A [ \boldX ]$ and $A'  [\boldX]$ are disconnected in
the graph $G_{r, R}$.
\end{enumerate}
\end{lemma}

\begin{proof} The proof of the first claim is similar to the
noiseless setting, except that the probability mass inside
a ball now has contributions from both the manifold and the
background clutter. For $x \in S_{\sigma - r}$,
the probability mass of
the ball $\Ball (x,r)$ under $Q$ is at most
$\pi v_d r^d (1 + \epsilon/6) \cdot \lambda (1 - \epsilon)
+ (1 - \pi) v_D r^D$, which is at most
$\frac{k}{n} - \frac{C_\delta}{n} \sqrt{k\mu}$.
Thus $x$ is removed during the cleaning step.
Similarly, if $x \notin M_r$, the ball $\Ball (x,r)$ does
not intersect the manifold, and hence its mass is at most
$(1 - \pi) v_D r^D$. Hence all points outside $M_r$ are
removed. Finally, if $x \in (A_{M, \sigma - r} \cup
A'_{M, \sigma - r}) \cap \boldX$, then the mass of
the ball $\Ball_M (x, r)$ is at least
$v_d r^d (1 - \epsilon/6) \lambda$ (ignoring the
contribution of the noise). This is at least
$\frac{k}{n} + \frac{C_\delta}{n} \sqrt{k\mu}$, and
hence $x$ is kept.

To prove the second claim, suppose that sets
$A \cap \boldX$ and $A' \cap \boldX$ are
connected in $G_{r, R}$. Then there exists
a sequence of sample points $y_0, y_1, \ldots,
y_t$ such that $y_0 \in A$, $y_t \in A'$ and
$\dist (y_{i-1}, y_{i}) \leq R$ for all
$1 \leq i \leq t$. Let $x_i$ be the projection
of $y_i$ on $M$, i.e., $x_i$ is the point of $M$
closest to $y_i$. We have already showed that
each $y_i$ lies inside the tube $M_r$, so
$\dist (x_i, y_i) \leq r$, and hence by triangle
inequality, we have $\dist (x_{i-1}, x_{i})
\leq R + 2r \leq \tau/4$. Hence, the
geodesic distance between $x_{i-1}$ and $x_i$ is
$< 2(R + 2r)$. Now, by an argument analogous to the
noiseless setting, there exists
a pair $(x_{i-1}, x_i)$ which are at a (geodesic)
distance at least $2(\sigma - r)$. This is a
contradiction since our parameter setting implies
that $2(\sigma - r) \geq 2(R + 2r)$.
\end{proof}

\begin{lemma} [Connectedness] \label{lemma:noisy-connectedness}
Assume that the parameters
$k, r$ and $R$ satisfy the separation
conditions (in Lemma \ref{lemma:noisy-separation}).
Then, with probability at least $1 - \delta$,
$A \cap \boldY$ is connected in $G_{r,R}$.
\end{lemma}

\begin{proof}
The proof of this lemma is identical to Lemma \ref{lemma:connectedness}
and is omitted.
\end{proof}

We now show how to pick the parameters to satisfy the conditions in 
Lemma \ref{lemma:noisy-separation}. Set
 $k := 144 C_\delta^2 (\mu/\epsilon^2)$, and define $r$ by 
\[
\pi v_d r^d (1 - \epsilon/6) \cdot \lambda = \frac{k}{n} + \frac{C_\delta}{n} \sqrt{k\mu}.
\]
It is easy to check that this setting satisfies all our requirements,
provided that the term $(1 - \pi) v_D r^D$ arising from the clutter 
noise satisfies the additional constraint 
\[
(1 - \pi) v_D r^D \leq (\epsilon/2) \times \pi v_d r^d \lambda.
\]
The definition of $r$ implies that $r$ is upper bounded by 
$\Big(\frac{2k}{n \lambda \pi v_d} \Big)^{1/d}$. Thus it 
suffices to ensure that 
\[
(1 - \pi) v_D \left(\frac{2k}{n \lambda \pi v_d} \right)^{D/d} \leq (\epsilon/2) \cdot \frac{2k}{n} = \frac{k \epsilon}{n}.
\]
This is equivalent to the condition 
\[
\lambda \geq  \frac{2 v_D^{d/D}}{v_d \epsilon^{d/D}} \cdot \frac{(1 - \pi)^{d/D}}{\pi} \cdot \left( \frac{k}{n} \right)^{1 - d/D},
\]
which is assumed by Theorem \ref{thm::clutter}.


\subsection{Proof of Theorem \ref{thm::additive-noise}}

Let $P$ be a distribution on a manifold $M$ with density $f$. 
Let $\boldX = (X_1, \ldots, X_n)$ be the latent sample from $P$, and let
$\boldY = (Y_1, \ldots, Y_n)$  be the observed sample. The only fact 
that we use about the observed sample is that it is close to the 
corresponding latent sample point: $\dist (Y_i, X_i) \leq \theta$, where
$\theta$ is the \emph{noise radius}. We show that we can adapt the 
RSL algorithm to resolve $(\sigma, \epsilon)$ separated clusters 
$(A, A')$, provided that $\theta$ is sufficiently small compared 
to both $\sigma$ and $\epsilon$. 

Again, we will pick values for $k, r, R$ based on a parameter $\rho$, 
defined as 
$\rho := \min (\frac{\sigma}{7}, \frac{\tau}{24}, \frac{\epsilon \tau}{144d})$. 

\begin{lemma} [Separation] \label{lemma:additive-separation}
Suppose $k, r, R$ are chosen to satisfy
\begin{align*}
\theta \leq r/2 ~~~~~~ r \leq \rho ~~~~~& R := 5 \rho,
\\
v_d (r - 2 \theta)^d (1 - \epsilon/6) \cdot \lambda &\geq
\frac{k}{n} + \frac{C_\delta}{n} \sqrt{k\mu} , \\
v_d (r + 2 \theta)^d (1 + \epsilon/6) \cdot \lambda (1 - \epsilon) &\leq
\frac{k}{n} - \frac{C_\delta}{n} \sqrt{k\mu} , \\
\end{align*}
Then, with probability $1-\delta$, the following holds uniformly 
over all $(\sigma, \epsilon)$ separated clusters $(A, A')$: 
\begin{enumerate}
\item If a latent sample point $X_i \in A_{M, \sigma - r + 2 \theta} \cup 
A'_{M, \sigma - r + 2 \theta}$, then the corresponding sample point $Y_i$ is 
kept during the cleaning step. If $X_i \in S_{M, \sigma - r - 2 \theta}$, 
then $Y_i$ is removed.  

\item The sets $\{ Y_i \,:\, X_i \in A \}$ and 
$\{ Y_i \,:\, X_i \in A' \}$ are disconnected in the graph $G_{r, R}$. 
\end{enumerate}
\end{lemma}

\begin{proof}
To prove the first claim, suppose $X_i \in A_{\sigma - r + 2\theta} \cup
A'_{\sigma - r + 2 \theta}$. Consider the ball 
$\Ball_M(X_i, r - 2 \theta)$. It is 
completely inside $A_{M, \sigma} \cup A'_{M, \sigma}$, hence the 
density $f$ inside it is at least $\lambda$. Moreover, if 
$X_j$ is in $\Ball_M(X_i, r - 2 \theta)$, then by triangle
inequality, we have \[
\dist (Y_j, Y_i) \leq \dist (X_j, Y_j) + \dist (X_j, X_i) + \dist(Y_i, X_i)
\leq r.
\]
Hence the ball $\Ball (X_i, r)$ contains at least $k$ sample points, provided 
$\Ball_M (X_i, r - 2\theta)$ contains at least $k$ points from $\boldX$. 
Finally, the true mass of the set $\Ball_M (X_i, r - 2\theta)$ is at least 
\[
v_d (r - 2 \theta)^d (1 - \epsilon/6) \cdot \lambda \geq 
\frac{k}{n} + \frac{C_\delta}{n} \sqrt{k\mu} .
\]
Hence it contains at least $k$ latent sample points, and 
we are done. 

Similarly, suppose $X_i \in S_{\sigma - r - 2\theta}$, and consider the 
ball $\Ball_M (X_i, r+2 \theta)$. It is completely contained inside 
$S_{M, \sigma}$ and hence the density inside 
the ball is at most $\lambda (1 - \epsilon)$. Moreover, if $X_j$ is outside 
the set, then \[
\dist (Y_j, Y_i) \geq \dist (X_j, X_j) - \dist (X_i, Y_i) - \dist(X_j, Y_j)
> r.
\]
Hence the ball $\Ball (Y_i, r)$ contains fewer than $k$ sample points, provided
$\Ball_M (X_i, r + 2\theta)$ contains fewer than $k$ points from $\boldX$. The
true mass of the ball $\Ball_M (X_i, r + 2\theta)$ is at most
\[
v_d (r + 2 \theta)^d (1 + \epsilon/6) \cdot \lambda (1 - \epsilon) \leq
\frac{k}{n} - \frac{C_\delta}{n} \sqrt{k\mu} .
\]
Hence the ball contains fewer than $k$ latent sample points, and
we are done. 

We now prove that the graph $G_{r, R}$ is disconnected. Suppose not. Then
there must exist a sequence of latent sample points 
$x_0, x_1, \ldots, x_t \in \boldY$ and a corresponding sequence of
noisy sample points $y_0, \ldots, y_t \in \boldX$ such that
$x_0 \in A$, $x_t \in A'$, and $\dist (y_{i-1}, y_i) \leq R$. Clearly 
$\dist(x_{i-1}, x_i) \leq R + 2 \theta \leq \tau/4$. Thus 
the geodesic distance between $x_{i-1}$ and $x_i$ is less than 
$2 ( R + 2\theta)$. However, by the $(\sigma, \epsilon)$ separation 
condition, we must have a successive pair $(x_{i-1}, x_i)$ whose 
geodesic distance is at least $2 (\sigma - r)$. This is a contradiction since
we have set our parameters such that $2 (\sigma - r) \geq 2 (R + 2 \theta)$. 
\end{proof}

\begin{lemma} [Connectedness] \label{lemma:additive-connectedness}
Assume that the conditions of Lemma \ref{lemma:additive-separation}
are satisfied. Then, with probability at least $1-\delta$, the
following holds uniformly over all $A$: if 
$\inf_{x \in A_{M, \sigma}} f(x) \geq \lambda$, then 
$\{ Y_i \,:\, X_i \in A \}$ is connected
in $G_{r, R}$. 
\end{lemma}

\begin{proof} The proof is similar to that of Lemma 
\ref{lemma:connectedness}, so we indicate only the necessary modifications,
omitting the details. We now use a net of radius 
$(R - 2 \theta)/4$, and the condition that $R \geq 4r$ is replaced by 
$R - 2 \theta \geq 4r$. Finally, the $x_i$'s defined in the proof are 
latent sample points, whereas the algorithm observes 
an arbitrary point $y_i$ in a $\theta$-ball around the $x_i$. 
Thus, the distance between $y_{i-1}$ and $y_i$ is at most
\[
4 \cdot \frac{R-2 \theta}{4} + \dist(y_i, x_i) + \dist(y_{i-1}, x_{i-1}) \leq R
.\]
\end{proof} 

In order to satisfy the conditions stated in Lemma \ref{lemma:additive-separation},
we need the assumption that $\theta$ is small compared to $r$. More precisely,
we will assume that $\theta \leq r \epsilon / 24 d$. Under this assumption, we can
satisfy the above conditions by ensuring that 
\begin{align*}
v_d r^d (1 - \epsilon/12) (1 - \epsilon/6) \cdot \lambda &\geq
\frac{k}{n} + \frac{C_\delta}{n} \sqrt{k\mu} , \\
v_d r^d (1 + \epsilon/6) (1 + \epsilon/6) \cdot \lambda (1 - \epsilon) &\leq
\frac{k}{n} - \frac{C_\delta}{n} \sqrt{k\mu} 
\end{align*}
As before, we can satisfy these equations by 
setting $k := O( C_\delta^2 \mu/\epsilon^2)$, and $r$ according to 
\[
v_d r^d (1 - \epsilon/12) (1 - \epsilon/6) \cdot \lambda =
\frac{k}{n} + \frac{C_\delta}{n} \sqrt{k\mu} .
\]
 
\subsection{Connection radius for polynomially bounded densities} 
\label{app::poly}
In this section, we prove that in our algorithm (Figure \ref{fig::CD}),
we can pick the connection radius $R$ to be $R := 4r$,
independent of the other parameters, provided that the
density level satisfies $\lambda \leq n^A$ for some
absolute constant $A$. (Our original
setting picked $R = 4 \rho$ and $r \leq \rho$.)

More precisely, we will argue 
that the parameter $\mu$ in the algorithm 
can be safely replaced by a related parameter 
$\tilde\mu := 2A \log n$ 
without affecting
the performance of the algorithm. Pick 
$k = O(C_{\delta}^2 \tilde{\mu} / \epsilon^2)$, and set $r, R$ by the equations
\begin{align*}
v_d r^d \lambda &= \frac{1}{1 - \epsilon/6}
\left( \frac{k}{n} + \frac{C_2 \log (1/\delta)}{n} \sqrt{k \tilde\mu}\right),
\\
R &= 4r.
\end{align*}

The crucial ingredient in the analysis of 
our algorithm is the uniform convergence property
of balls centered at the sample points and net points
(Lemma \ref{lemma:uniform-convergence}), so we first 
verify that this statement remains true. Note that by 
our choice of $r$, we have
\begin{align*}
v_d r^d \lambda \geq \frac{k}{n} \geq \frac{1}{n}, 
\end{align*}
so that $1/r^d \leq v_d n \lambda \leq v_d n^{A+1} \leq n^{A+1}$ 
(since $v_d < 1$ for sufficiently large $d$). As before,
we consider a net $\Net$ of radius $R/4$ (i.e., $r$); by Lemma
\ref{lemma:coversize}, size of this net is at most $c^d/r^d$ for some 
absolute constant
$c > 0$. Thus
by Lemma \ref{lemma:uniform-convergence}, we have the uniform
convergence property, provided the 
parameter $\mu$ is replaced by  
\begin{align*}
\log n + \log |\Net| = \log n + \log (1/r^d) + O(1) = (A+2) \log n + O(1).
\end{align*}
Notice that $\tilde \mu$ is picked to be a safe upper bound on this quantity, hence
the lemma holds when $\mu$ is replaced by $\tilde\mu$. 

Finally, it is easy to check that our choice of parameters
satisfies all the conditions given in the separation lemma. Hence 
the separation and connectedness guarantees
(Lemmas \ref{lemma:separation} and \ref{lemma:connectedness}), 
together with their proofs, remain unaffected.